\documentclass[twoside,leqno,twocolumn]{article}

\usepackage[letterpaper]{geometry}

\usepackage{ltexpprt}
\usepackage[inline]{enumitem}
\usepackage[utf8]{inputenc} % allow utf-8 input
\usepackage[T1]{fontenc}    % use 8-bit T1 fonts
\usepackage{hyperref}       % hyperlinks
\usepackage{url}            % simple URL typesetting
\usepackage{booktabs}       % professional-quality tables
\usepackage{amsfonts}       % blackboard math symbols
\usepackage{nicefrac}       % compact symbols for 1/2, etc.
\usepackage{microtype}      % microtypography

\usepackage{hyperref}
\usepackage{url}
\usepackage{float}
\usepackage{graphicx}
\usepackage{amssymb}
\usepackage{comment}
\usepackage{enumitem}
\usepackage{booktabs}
\usepackage{xspace}
\usepackage[dvipsnames]{xcolor}
\usepackage{subcaption}
\usepackage{threeparttable}
\usepackage{amsmath,amsfonts,bm}
\def\cA{{\mathcal{A}}}

\def\eqref#1{equation~\ref{#1}}
\def\1{\bm{1}}

\def\va{{\bm{a}}}

\def\vh{{\bm{h}}}

\def\mA{{\bm{A}}}

\def\mE{{\bm{E}}}

\def\mI{{\bm{I}}}

\def\mW{{\bm{W}}}
\def\mX{{\bm{X}}}
\def\mY{{\bm{Y}}}

\DeclareMathAlphabet{\mathsfit}{\encodingdefault}{\sfdefault}{m}{sl}
\SetMathAlphabet{\mathsfit}{bold}{\encodingdefault}{\sfdefault}{bx}{n}

\def\cA{{\mathcal{A}}}

\def\cG{{\mathcal{G}}}

\def\cO{{\mathcal{O}}}
\def\cP{{\mathcal{P}}}

\def\cX{{\mathcal{X}}}
\def\cY{{\mathcal{Y}}}

\newcommand{\R}{\mathbb{R}}

\DeclareMathOperator*{\argmax}{arg\,max}

\newcommand{\prelim}[1]{\noindent{\bf (#1)}}

\newcommand{\textBF}[1]{\fontseries{b}\selectfont #1}

\usepackage{cleveref}

\begin{document}

\title{Learning over Families of Sets - Hypergraph Representation Learning for Higher Order Tasks}

\author{
Balasubramaniam Srinivasan \\ Purdue University \\ bsriniv@purdue.edu \and Da Zheng \\ Amazon Web Services \\ dzzhen@amazon.com	 \and George Karypis \\ Amazon Web Services\\ gkarypis@amazon.com
}

\date{}

\maketitle
% Copyright Statement
% When submitting your final paper to a SIAM proceedings, it is requested that you include 
% the appropriate copyright in the footer of the paper.  The copyright added should be 
% consistent with the copyright selected on the copyright form submitted with the paper.
% Please note that "20XX" should be changed to the year of the meeting.

% Default Copyright Statement

\fancyfoot[R]{\scriptsize{Copyright \textcopyright\ 2021 by SIAM\\
Unauthorized reproduction of this article is prohibited}}

%abstract
\begin{abstract}
Graph representation learning has made major strides over the past decade.
However, in many relational domains, the input data are not suited for simple graph representations as the relationships between entities go beyond pairwise interactions.
In such cases, the relationships in the data are better represented as hyperedges (set of entities) of a non-uniform hypergraph. 
%for example, email communication data, question answer forums, drug composition information, etc.
While there have been works on principled methods for learning representations of nodes of a hypergraph, these approaches are limited in their applicability to tasks on non-uniform hypergraphs (hyperedges with different cardinalities).
In this work,
we exploit the incidence structure to develop a hypergraph neural network to learn provably expressive representations of variable sized hyperedges which preserve local-isomorphism in the line graph of the hypergraph, while also being invariant to permutations of its constituent vertices.
Specifically, for a given vertex set, we  propose frameworks for (1) hyperedge classification and (2) variable sized expansion of partially observed hyperedges which captures the higher order interactions among vertices and hyperedges.
We evaluate performance on multiple real-world hypergraph datasets and demonstrate consistent, significant improvement in accuracy, over state-of-the-art models.
\end{abstract}

% \vspace{-0.1in}

%introduction
\section{Introduction}
\label{sec:into}

Deep learning on graphs has been a rapidly evolving field due to its widespread applications in domains such as e-commerce, personalization, fraud \& abuse, life sciences, and social network analysis.
However, graphs can only capture interactions involving pairs of entities whereas in many of the aforementioned domains any number of entities can participate in a single interaction. 
% Additionally, there may not necessarily exist a total or partial ordering 
%of the interacting entities. 
For example, more than two substances can interact at a specific instance to form a new compound, study groups can contain more that two students, recipes contain multiple ingredients, shoppers purchase multiple items together, etc.
Graphs, therefore, can be an over simplified depiction of the input data (which may result in loss of significant information).
Hypergraphs, \cite{berge1984hypergraphs} (see \Cref{fig:hypergraph}(a) for example), which serve as the natural extension of dyadic graphs, form the obvious solution. 
% Since graphs are conveniently represented as an edge set, the hypergraph is therefore correspondingly represented as a family of sets, where each set in the family models an interaction event.

\begin{figure*}[ht!!!]
\vspace{-0.7in}
\begin{center}
\centerline{\includegraphics[height=2.5in, width=5.5in]{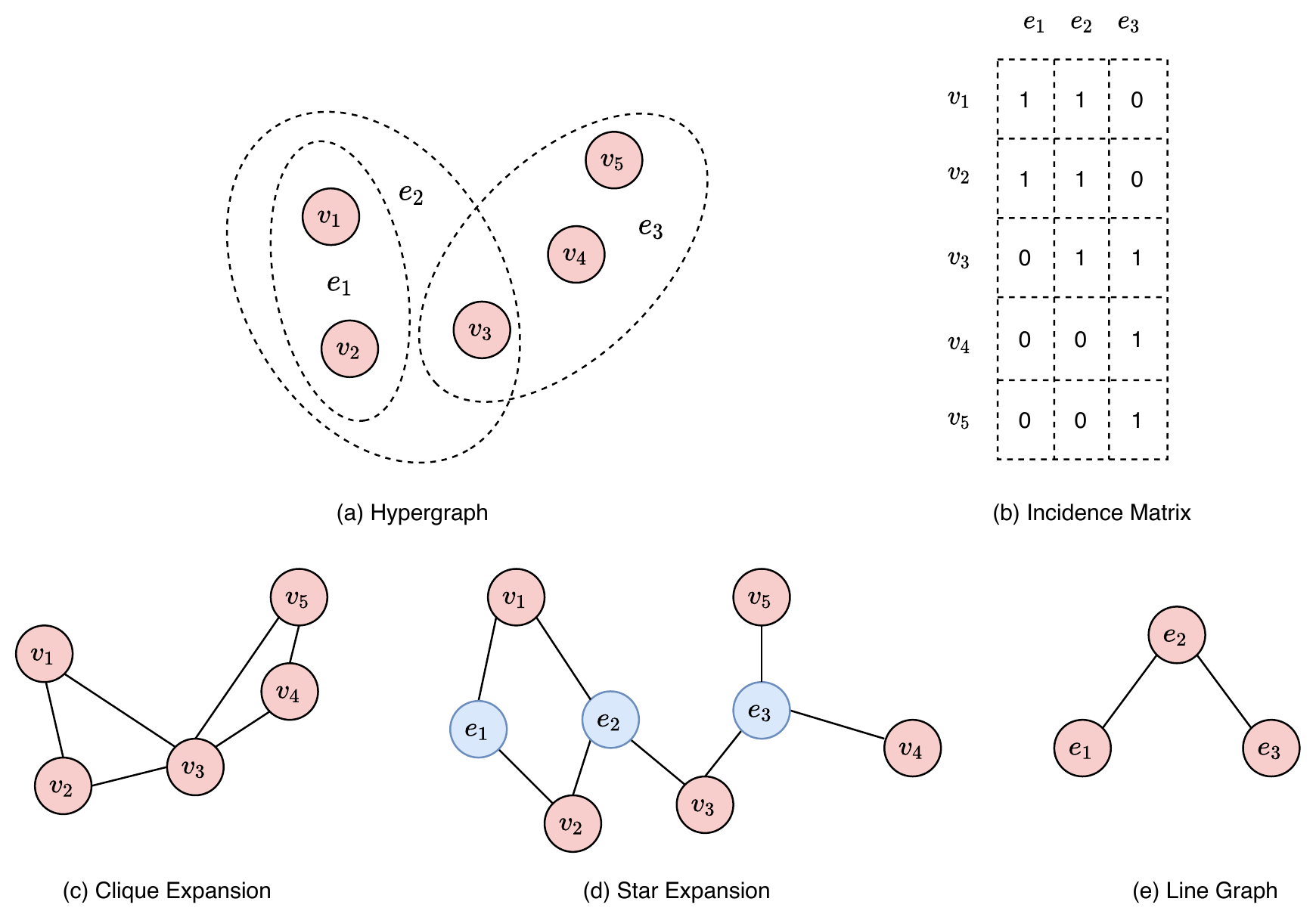}}
\caption{\small A Hypergraph(a) with 5 nodes $v_1, v_2, \ldots v_5$ and 3 hyperedges $e_1=\{v_1,v_2\}, e_2=\{v_1,v_2,v_3\}, e_3=\{v_3,v_4,v_5\}$ , its incidence matrix(b), its clique expansion (c), its star expansion (d) and its line graph(e)}
\label{fig:hypergraph}
\end{center}
\vspace{-0.4in}
\end{figure*}

Due to the ubiquitous nature of hypergraphs, learning on hypergraphs has been studied for more than a decade \cite{agarwal2006higher, zhou2007learning}.
Early works on learning on hypergraphs employed random walk procedures~\cite{lu2011high, bellaachia2013random, chitra2019random} and the vast majority of them were limited to hypergraphs whose hyperedges have the same cardinality ($k$-uniform hypergraphs). 
More recently, with the growing popularity and success of message passing graph neural networks \cite{kipf2016semi, hamilton2017inductive}, message passing hypergraph neural networks learning frameworks have been proposed \cite{feng2019hypergraph, bai2019hypergraph, yadati2019hypergcn, zhang2019hyper, yang2020hypergraph}.
These works rely on constructing the clique expansion (\Cref{fig:hypergraph}(c)), star expansions (\Cref{fig:hypergraph}(d)), or other expansions of the hypergraph that preserve partial information.
Subsequently, node representations are learned using GNN's on the graph constructed as a proxy of the hypergraph. 
These strategies are insufficient as either (1) there does not exist a bijective transformation between a hypergraph and the constructed clique expansion (loss of information); (2) they do not accurately model the underlying dependency between a hyperedge and its constituent vertices (for example, a hyperedge may cease to exist if one of the nodes were deleted); (3) they do not directly model the interactions between different hyperedges.
The primary goal of this work is to address these issues and to build models which better represent hypergraphs.

Corresponding to the adjacency matrix representation of the edge set of a graph, a hypergraph is commonly represented as an incidence matrix (\Cref{fig:hypergraph}(b)), in which a row is a vertex, a column is a hyperedge and an entry in the matrix is 1 if the vertex belongs to the hyperedge.
% An incidence matrix is a space-efficient representation of a hypergraph in comparison to a more tedious representation using multiple tensors of different sizes (the number of tensors being dictated by the number of distinct hyperedge sizes).
In this work, we directly seek to exploit the incidence structure of the hypergraph to learn representations of nodes and hyperedges.
Specifically, for a given partially observed hypergraph, we synchronously learn vertex and hyperedge representations that simultaneously take into consideration both the line graph \Cref{fig:hypergraph}(e) and the set of hyperedges that a vertex belongs to in order to learn provably expressive representations.
% of observed and unobserved hyperedges.
The jointly learned vertex and hyperedge representations are then used to tackle higher-order tasks such as expansion of partially observed hyperedges and classification of unobserved hyperedges. 

While the task of hyperedge classification has been studied before, set expansion for relational data has largely been unexplored.
For example, given a partial set of substances which are constituents of a single drug, hyperedge expansion entails completing the set of all constituents of the drug while having access to composition to multiple other drugs.
A more detailed example for each of these tasks is presented in the Appendix - \Cref{subsection:appendix}.
For the hyperedge expansion task, we propose a GAN framework \cite{goodfellow2014generative} to learn a probability distribution over the vertex power set (conditioned on a partially observed hyperedge), which maximizes the point-wise mutual information between a partially observed hyperedge and other disjoint vertex subsets in the vertex power set.

{\em Our  Contributions can be summarized as:}
(1) Propose a hypergraph neural network which exploits the incidence structure and hence works on real world sparse hypergraphs which have hyperedges of different cardinalities.
(2) Provide provably expressive representations of vertices and hyperedges, as well as that of the complete hypergraph which preserves properties of hypergraph isomorphism.
(3) Introduce a new task on hypergraphs -- namely the variable sized hyperedge expansion and also perform variable sized hyperedge classification.
Furthermore, we demonstrate improved performance over existing baselines on majority of the hypergraph datasets using our proposed model.

%preliminaries
\section{Preliminaries}
\label{sec:prelims}

%Here, for the convenience of the reader, we introduce commonly used terms regarding hypergraphs and group theoretic literature. We then follow this up with a mathematical formulation of the problem.

In our notation henceforth, we shall use capital case characters (e.g., $A$) to denote a set or a hypergraph, bold capital case characters (e.g., $\mA$) to denote a matrix, and capital characters with a right arrow over it (e.g., $\overrightarrow{A}$) to denote a sequence with a predefined ordering of its elements.
We shall use lower characters (e.g., $a$) to denote the element of a set and bold lower case characters (e.g., $\va$) to denote vectors.
Moreover, we shall denote the $i$-th row of a matrix $\mA$ with $\mA_{i\cdot}$, the $j$-th column of the matrix with $\mA_{\cdot j}$, and use $A_m$ to denote a subset of the set $A$ of size $m$ i.e., $A_m \subseteq A; \, |A_m| = m$.

\prelim{Hypergraph}
Let $H=(V,E, \mX, \mE)$ denote a hypergraph $H$ with a finite vertex set
$V=\{v_1, \ldots , v_n\}$, corresponding vertex features
$\mX \in \R^{n \times d}; \, d>0$, a finite hyperedge set
$E = \{e_1, \ldots, e_m\}$, where $E \subseteq P^*(V) \setminus \{\emptyset\}$
and $\bigcup\limits_{i=1}^{m}e_i=V$, where $P^*(V)$ denotes the power set on
the vertices, the corresponding hyperedge features $\mE \in \R^{m \times d}; \, d > 0$.
We use $E(v)$ (termed star of a vertex) to denote the hyperedges incident on a vertex $v$ and use $S_H$, a set of tuples, to denote the family of stars where $S_H = \{(v, E(v)): \forall v \in V\}$
called the family of stars of H.
When explicit vertex and hyperedge features and weights are unavailable, we will consider  $\mX=\mathbf{1_n}\mathbf{1_n}^T$, $\mE=\mathbf{1_m}\mathbf{1_m}^T$, where $\mathbf{1}$ represents a $n \times 1$ or $m \times 1$ vector of ones respectively.
% , and a weight vector $\vw \in \R^m$ composed of scalar values associated with edge hyperedge.
The vertex and edge set $V,E$ of a hypergraph can equivalently be represented with an incidence matrix $\mI \in \{0,1\}^{|V| \times |E|}$, where $\mI_{ij}=1$ if $v_i \in e_j$ and $\mI_{ij}=0$ otherwise.
Isomorphic hypergraphs either have the same incidence matrix or a row/column/row and column permutation of the incidence matrix i.e., the matrix $\mI$ is separately exchangeable.
We use $L_H$ to denote the line graph (\Cref{fig:hypergraph}(e)) of the hypergraph, use $H^\star$ to denote the dual of a hypergraph.
% , and $C_H$ to denote the clique expansion. 
Additionally, we define a function $LG_{H}$ , a multi-valued function termed line hypergraph of a hypergraph - which generalizes the concepts line graph and the dual of a hypergraph and defines the spectrum of values which lies between them.
For the scope of this work, we limit ourselves for $LG_{H}$ to be a dual valued function - using only the two extremes, such that $LG_{H}(0) = L_H$ and  $LG_{H}(1) = H^\star$.

Also, we use $\Sigma_{n, m}$ to denote the set of all possible attributed hypergraphs $H$ with $n$ nodes and $m$ hyperedges. More formally, $\Sigma_{n, m}$ is the set of all tuples $(V,E,\mX, \mE)$  --- for vertex node set size $n$ and hyperedge set size $m$.

\prelim{1-Weisfeiler-Leman(1-WL) Algorithm}
Let $G=(V,E)$ be a graph, with a finite vertex set $V$ and let $s:V \to \Delta$ be a node coloring function with an arbitrary co-domain $\Delta$ and $s(v), v \in V$ denote the color of a node in the graph.
Correspondingly, we say a labeled graph $(G, s)$ is a graph $G$ with a complete node coloring $s:V \to \Delta$.
The 1-WL Algorithm \cite{babai1980random} can then be described as follows:
let $(G, s)$ be a labeled graph and in each iteration, $t \geq 0,$ the 1-WL computes a node coloring $c_{s}^{(t)}: V \to \Delta,$ which depends on the coloring from the previous iteration.
The coloring of a node in iteration $t>0$ is then computed as
$c_{s}^{(t)}(v)= \text{HASH}\left(\left(c_{s}^{(t-1)}(v),\left\{c_{s}^{(t-1)}(u) \mid u \in N(v)\right\}\right)\right)$
where HASH is bijective map between the vertex set and $\Delta$, and $N(v)$ denotes the 1-hop neighborhood of node $v$ in the graph.
The 1-WL algorithm terminates if the number of colors across two iterations does not change, i.e., the cardinalities of the images of $c_{s}^{(t-1)}$ and $c_{s}^{(t)}$ are equal.
The 1-WL isomorphism test, is an isomorphism test, where the 1-WL algorithm is run in parallel on two graphs $G_1, G_2$ and the two graphs are deemed non-isomorphic if a different number of nodes are colored as $\kappa$ in $\Delta$.

\prelim{Graph Neural Networks (GNNs)}
For a graph $G=(V,E,\mX)$, modern GNNs use the edge connectivity and node features $\mX$ to learn a representation vector of a node, $h_v$, or the entire graph, $h_G$.
They employ a neighborhood aggregation strategy, where  the representation of a node is iteratively updated by an aggregation of the representations of its neighbors. 
Multiple layers are employed to capture the k-hop neighborhood of a node.
The update equation of a GNN layer can be written as
{\small
$$ h_v^{k} = \text{COMBINE}(h_v^{k-1},\text{AGGREGATE}^{k}(\{h_u^{k-1}: u \in N(v)\}))$$
 }
where $h_v^k$ is the representation of node $v$ after k layers and $N(v)$ is the 1-hop neighborhood of $v$ in the graph.
\cite{morris2019weisfeiler, xu2018powerful} showed that message passing GNNs are no more powerful than the 1-WL algorithm.

\prelim{Finite Symmetric Group $\mathbb{S}_n$}
A finite symmetric group $\mathbb{S}_m$ is a discrete group $\cG$ defined over a finite set of size $m$ symbols (w.l.o.g. we shall consider the set $\{1,2,\ldots, m\}$) and consists of all the permutation operations that can be performed on the $m$ symbols.
Since the total number of such permutation operations is $m!$ the order of $\mathbb{S}_m$ is m!.

\prelim{Group Action (left action)}
If $\cA$ is a set and $\cG$ is a group, then $\cA$ is a $\cG$-set if there is a function $\phi: \cG \times \cA \to \cA$, denoted by $\phi(g,a) \mapsto ga$, such that:
\begin{enumerate}[label=(\roman*), noitemsep]
    \item $1a = a$ for all $a \in \cA$, where $1$ is the identity element of the group $\cG$
    \item $g(ha) = (gh)(a)$ for all $g,h \in \cG$ and $a \in \cA$ 
\end{enumerate}

\prelim{Orbit}
Given a group $\cG$ acting on a set $\cA$, the orbit of an element $a \in \cA$ is the set of elements in $\cA$ to which $a$ can be moved by the elements of $\cG$. 
The orbit of $a$ is denoted by $\cO(a) = \{g \cdot a | g \in \cG\} \subset \cA$.

\prelim{Vertex Permutation action $\pi_V$}
A vertex permutation action $\pi \in \mathbb{S}_k$ is the application of a left action $\phi: \mathbb{S}_k \times V_k \to V_k$ with the element $\pi$ on a sorted sequence of $k$ vertices represented as $\overrightarrow{V_k} = (v_1, \ldots,v_k)$ of a hypergraph to output a corresponding permuted sequence of vertices i.e., $\phi(\pi,\overrightarrow{V_k}) = \overrightarrow{V_{k_{\pi}}} = (v_{\pi(1)}, \ldots, v_{\pi(k)})$.
A permutation action $\pi \in \mathbb{S}_n$ can also act on any vector, matrix, or tensor defined over the nodes $V$, e.g., $(\mX_{i\cdot})_{i \in V}$, and output an equivalent vector, matrix, or tensor with the order of the nodes permuted e.g., $(\mX_{\pi(i)\cdot})_{\pi(i) \in V}$. 
% It is important to note that the permutations of the vertices does not affect the permutations of the edgeset.

\prelim{Hyperedge Permutation Action $\pi_E$}
A hyperedge permutation action $\pi \in \mathbb{S}_k$ is the application of a left action $\psi: \mathbb{S}_k \times E_k \to E_k$ with the element $\pi$ on a sorted sequence of $m$ hyperedges represented as $\overrightarrow{E_k} = (e_1, \ldots,e_k)$ of a hypergraph to output a corresponding permuted sequence of hyperedges i.e., $\psi(\pi,\overrightarrow{E_k}) = \overrightarrow{E_{k_{\pi}}} = (e_{\pi(1)}, \ldots, e_{\pi(k)})$.
A permutation action $\pi \in \mathbb{S}_m$ can also act on any vector, matrix, or tensor defined over the hyperedges $E$, e.g., $(\mE_{i\cdot})_{i \in E}$, and output an equivalent vector, matrix, or tensor with the order of the hyperedges permuted e.g., $(\mE_{\pi(i)\cdot})_{\pi(i) \in E}$.
It is crucial to note that a vertex permutation action can be simultaneously performed along with the hyperedge permutation.
We represent a joint permutation on the entire edge set $E$ as $\pi_2(\pi_1(E))$, and for a hyperedge $e \in E$ as $\pi_2(\pi_1(e))$ where $\pi_i \in \mathbb{S}_n, \pi_2 \in \mathbb{S}_m$

\prelim{Node Equivalence Class/ Node Isomorphism}
The equivalence classes of vertices  $v \in V$ of a hypergraph $H$ under the action of automorphisms between the vertices are called vertex equivalence classes or vertex orbits.
If two vertices $v_1,v_2 \in V$ are in the same vertex orbit, we say they are node isomorphic and are denoted by $v_1 \cong v_2$.

\prelim{Hyperedge Orbit/ Isomorphism}
The equivalence classes of non empty subsets of vertices  $e \in \cP^\star(V) \setminus \emptyset; \, e \in E$ of a hypergraph $G$ under the action of automorphisms between the subsets are called hyperedge orbits.
If two proper subsets $e_1,e_2 \in P^*(V) \setminus \emptyset $ are in the same hyperedge orbit, we say they are hyperedge isomorphic and are denoted by $e_1 \cong e_2$.

\prelim{Hypergraph Orbit and Isomorphism}
The hypergraph orbit of a hypergraph $H$, given by 
% $\mathbb{S}_n \cdot H = \{\pi \cdot G | \pi \in \mathbb{S}_n\}$ is given by 
the application of the elements $\pi$ of the finite symmetry group $\mathbb{S}_n$ on the vertex set $V$/ $\mathbb{S}_m$ on the edge set $E$/ or any combination of the two and appropriately modifying the 
associated matrices $\mX, \mE, \mI$ of the hypergraph.
% hyperedge set $E$ i.e., for example, if $e_i = \{v_1, v_2, \ldots, v_m\}$, the application of $\pi$ on the vertex set, would correspondingly alter $e_i$ as $ e_i = \{v_{\pi(1)}, v_{\pi(2)}, \ldots, v_{\pi(m)}\}$.
Two hypergraphs $H_1$ and $H_2$ are said to be isomorphic (equivalent) denoted by $H_1 \cong H_2$ iff there exists either a vertex permutation action or hyperedge permutation action or both such that $H_1 = \pi \cdot H_2$.
The hypergraph orbits are then the equivalence classes under this relation; two hypergraphs $H_1$ and $H_2$ are equivalent iff their hypergraph orbits are the same.
% , that is, $\mathbb{S}_n \cdot H_1 = \mathbb{S}_n \cdot H_2$. 

\prelim{$\mathcal{G}$-invariant functions}
A function $\phi$ acting on a hypergraph $H$ given by $\phi:\Sigma_{n,m} \to \R^{\bm{\circ}}$ is $\mathcal{G}$-invariant whenever it is invariant to any vertex permutation/ edge permutation action $\pi \in \mathbb{S}_{\cdot}$ in the $\Sigma_{n,m}$ symmetric space i.e., $\phi(\pi \cdot H) = \phi(H)$ and all isomorphic hypergraphs obtain the same representation.
Similarly, a function $\rho: P^*(V) \setminus \emptyset \to R^{\cdot} $acting on a hyperedge for a given hypergraph $H$, is said to be $\cG$-invariant iff all isomorphic hyperedges obtain the same representation.

% \prelim{Local k-$\delta$-Weisfeiler-Leman Algorithm}
% \vspace{-0.1in}
\subsection{Problem Setup}
% Next, we mathematically formulate the problem of partially observed hyperedge expansion and as well as prediction of unobserved hyperedges in the hypergraph.
\paragraph{Partially observed hyperedge expansion}
Consider a hypergraph $H=(V,E',\mX,\mE)$ where  a small fraction of hyperedges in the hyperedge set are partially observed and let $E$ be the completely observed hyperedge set.
% and $|E'| = |E| = m$. 
A partially observed hyperedge implies $\exists e'_i \in E'$, $\exists v_j  \in V, v_j \not\in e'_i$ but $v_j \in e_i, \, e_i \in E$, where $e_i$ is the corresponding completely observed hyperedge of $e'_i$ 
The task here is, given a partial hyperedge $e' \in E', e' \not\in E$ but $e' \subset e, e \in E$, to complete $e'$ with vertices from $V$ so that after hyperedge expansion $e'=e$.

\paragraph{Unobserved hyperedge classification}
Consider a hypergraph $H=(V,E',\mX,\mE')$ with an incompletely observed hyperedge set $E'$ and let $E$ be the corresponding completely observed hyperedge set with $E' \subset E$.
An incomplete hyperedge set implies $\exists  e \in E; \, e \not\in E'$ where  $|E'| < |E| = m$.
It is important to note that in this case, if a certain hyperedge is present in $E'$, then the hyperedge is not missing any vertices in the observed hyperedges.
The task here is, for a given hypergraph $H$, to predict whether a new hyperedge $e$ was present but unobserved in the noisy hyperedge set i.e., $e \not\in E'$ but $e \in E$.
% \vspace{-0.1in}
\section{Learning Framework and Theory}
\label{sec:learning_theory}

Previous hypergraph neural networks \cite{feng2019hypergraph, yadati2019hypergcn, bai2019hypergraph}, employ a proxy graph to learn vertex representations for every vertex $v \in V$, by aggregating information over its neighborhood.
% Multiple message passing layers $(k)$ are then employed to obtain information over the $k-$hop neighborhood of the vertex.
Hyperedge representations (or alternatively, hypergraph) are then obtained, when necessary, by using a pooling operation (e.g. sum, max, mean, etc)  over the vertices in the hyperedge (vertex set of the hypergraph). 
However, such a strategy, fails to (1) preserve properties of equivalence classes of hyperedges/hypergraphs and (2) capture the implicit higher order interactions between the nodes/ hyperedges, and fails on higher order tasks as shown by \cite{Srinivasan2020On}.
% To alleviate this problem, and learn joint representations of vertex subsets we aggregate information from multiple positional embedding samples to learn approximate $\cG-$invariant structural representations of hyperedges (node sets), and also makes our procedure inductive. 

To alleviate these issues, in this work, we use a message passing framework on the incidence graph representation of the observed hypergraph, which synchronously updates the node and observed hyperedge embeddings as follows:

\begin{equation}
\small
\label{eq:edge_update}
\begin{split}
       \vh^k_e = \sigma(\mW^k_E \cdot (\vh^{(k-1)}_e \otimes f^k(\{(\vh^{(k-1)}_v \otimes p^k(\{\vh_{e'}^{(k-1)} \\
       : e' \ni v\}))   : \forall v \in e \, \text{where}\, v \in V, e,e' \in E\}) ) ) 
\end{split}
\end{equation}
\begin{equation}
   \small
   \label{eq:vertex_update}
   \begin{split}
   \vh^k_v = \sigma(\mW^k_V \cdot (\vh^{(k-1)}_v \otimes g^k(\{\vh^{(k-1)}_e \otimes q^k(\{\vh_{v'}^{(k-1)} \\
   : v' \in e\}))  : \forall e \ni v\ \, \text{where}\, v,v' \in V, e \in E\})))
   \end{split}
\end{equation}
where, $\otimes$ denotes vector concatenation, $f^k,g^k, p^k, q^k$ are injective set functions (constructed via \cite{zaheer2017deep, murphy2018janossy}) in the $k^{\text{th}}$ layer, $\vh^k_e, \vh^k_v$ are the vector representations of the hyperedge and vertices after $k$ layers, $\mW^k_V, \mW^k_E$ are learnable weight matrices  and $\sigma$ is an element-wise activation function.
We use $K$ (in practice, $K$=2) to denote the total number of convolutional layers used.
From \Cref{eq:vertex_update} it is clear to see that a vertex  not only receives messages from the hyperedges it belongs too, but also from neighboring vertices in the clique expansion.
Similarly, from  \Cref{eq:edge_update}, a hyperedge receives messages from its constituent vertices as well as neighboring hyperedges in the line graph of the hypergraph.

However, the above equations, standalone do not present a framework to learn representations of unobserved hyperedges for downstream tasks.
In order to do this, post the convolutional layers, the representation of any hyperedge (observed or unobserved) are obtained using a function $\Gamma : P^\star (V) \times \Sigma_{n,m} \to \R^d $ as:
\begin{equation}
\small
\label{eq:unobserved_struct}
\begin{split}
    \Gamma(e',H; \Theta) = \phi(\{h_{v_{i}}^K : v_i \in e'\}) \otimes  \rho(\{\vh_{e}^{K}: e \ni v_i \, \forall v_i \in e'\}) \,\\ \text{where}\, v_i \in V, e \in E'
\end{split}
\end{equation}
where $\phi, \rho$ are injective set, multiset functions respectively, and $\Theta$ denotes the model parameters of the entire hypergraph neural network (convolutional layers, set functions).
% and $\otimes$ is vector concatenation.
Correspondingly, the representation of the complete hypergraph is  obtained using a function $\Gamma : \Sigma_{n,m} \to \R^d $ as:
\begin{equation}
\label{eq:separate}
  \Gamma(H; \Theta) = \phi(\{h_v^K\,:\, v \in V\}) \otimes \rho(\{h_e^K\,:\,e \in E\}))
\end{equation}
% where $f,g$ are set functions, $\mW_H$ is a weight matrix, $b_H$ is a learnable bias and $\otimes$ is vector concatenation.

\subsection{Theory}
In what follows, we list the properties of the vertex/ hyperedge representations.
All proofs are presented in the Supplementary Material.
\label{sec_mini:theory}
\begin{property}[Vertex Representations]
The representation of a vertex $v \in V$ in a hypergraph $H$ learnt using $\Cref{eq:vertex_update}$ is a $\cG$-invariant representation $\Phi(v, V, E, \mX, \mE)$ where $\Phi:V \times \Sigma_{n,m} \to \R^d, d\geq 1$ such that  $\Phi(v, V, E, \mX, \mE)$ = $\Phi((\pi_1(v), \pi_1(V), \pi_2(\pi_1(E)), \pi_1(\mX), \pi_2(\pi_1(\mE)))$ $\forall \pi_1 \forall \pi_2$ where $ \pi_1 \in \mathbb{S}_n \, \text{and} \, \pi_2 \in \mathbb{S}_m$.  
Moreover, two vertices $v_1, v_2$ which belong to the same vertex equivalence class i.e. $v_1 \cong v_2$ obtain the same representation.
\end{property}
% \begin{proof}
% \end{proof}
% 
\begin{property}[Hyperedge Representations]
The representation of a hyperedge $e \in E$ in a hypergraph $H$ learnt using $\Cref{eq:edge_update}$ is a $\cG$-invariant representation $\Phi(e, V, E, \mX, \mE)$ where $\Phi:P^\star(V) \times \Sigma_{n,m} \to \R^d, d\geq 1$ such that  $\Phi(e, V, E, \mX, \mE)$ = $\Phi((\pi_2(\pi_1(e)), \pi_1(V), \pi_2(\pi_1(E)), \pi_1(\mX), \pi_2(\pi_1(\mE)))$ $\forall \pi_1 \forall \pi_2$ where $ \pi_1 \in \mathbb{S}_n \, \text{and} \, \pi_2 \in \mathbb{S}_m$
Moreover, two hyperedges $e_1, e_2$ which belong to the same hyperedge equivalence class i.e. $e_1 \cong e_2$ obtain the same representation.
\end{property}
% 
% \begin{proof}
% 
% \end{proof}
% 
% 
Next, we restate a theorem from \cite{tyshkevich1996line} which provides a means to deterministically distinguish non isomorphic hypergraphs.
Subsequently, we characterize the expressivity of our model to distinguish non-isomorphic hypergraphs.

\begin{theorem}[\cite{tyshkevich1996line}]
Let $H_1, H_2$ be hypergraphs without isolated vertices whose line hypergraphs ${LG}_{H_1}, {LG}_{H_2}$ are isomorphic. 
Then $H_{1} \cong H_{2}$ if and only if there exists a bijection $\beta: V {LG}_{H_{1}} \rightarrow V {LG}_{H_{2}}$ such that $\beta\left(S_{H_{1}}\right)=S_{H_{2}},$ where $S_{ H_{i}}$ is the family of stars of the hypergraph $H_{i}$
\end{theorem}
\begin{theorem}
Let $H_1$, $H_2$ be two non isomorphic hypergraphs with finite vertex and hyperedge sets and no isolated vertices.
If the Weisfeiler-Lehman test of isomorphism decides their line graphs $L_{H_1}, L_{H_2}$ or the star expansions of their duals  $H_1^\star, H_2^\star$ to be not isomorphic 
% or the clique expansions $C_{H_1}, C_{H_2}$ to be not isomorphic, 
then there exists a function $\Gamma: \Sigma_{n,m} \to \R^d$ (via \Cref{eq:separate}) and parameters $\Theta$ that maps the hypergraphs $H_1, H_2$ to different representations.
% Two non-isomorphic hypergraphs $H_1$, $H_2$, without isolated vertices are distinguishable  if either of their clique expansions or line graphs are 1-WL distinguishable.
% Isomorphic hypergraphs also obtain the same representation.
\end{theorem}
% 
% 
% \begin{proof}
% 
% \end{proof}
% 
We now, extend this to the expressivity of the hyperedge representations and then show that the property of separate exchangeability \cite{aldous1981representations} of the incidence matrix is preserved by the hypergraph representation.
\begin{corollary}
There exists a function  $\Gamma : P^\star (V) \times \Sigma_{n,m} \to \R^d $ (via \Cref{eq:unobserved_struct}) and parameters $\Theta$ that maps two non-isomorphic hyperedges $e_1, e_2$ to different representations.
\end{corollary}
% 
% \begin{proof}
% 
% \end{proof}
% 
\begin{remark}[Separate Exchangeability]
The representation of a hypergraph $H$ learnt using the function $\Gamma: \Sigma_{n,m} \to \R^d$ (via \Cref{eq:separate}) preserves the separate exchangeability of the incidence structure $\mI$ of the hypergraph.
\end{remark}
% 
% \begin{proof}
% 
% \end{proof}
% 

We now describe the learning procedures for the two tasks, namely variable size hyperedge classification and variable size hyperedge expansion.

\subsection{Hyperedge Classification}
% In this section, for ease of readability, we shall abuse notation and refer to a hypergraph as $H=(V,E)$, dropping the feature matrices.
For a hypergraph $H$, let $E'$ denote the partially observed hyperedge set in our data corresponding to the true hyperedge set $E$.
The goal here is to learn a classifier $r : \R^d \to \R$  over the representations of hyperedges (obtained using \Cref{eq:unobserved_struct})  {\em s.t}  $\,\sigma(r( \Gamma(\{v_i, v_2, \ldots, v_M\}, H)))$ is used to classify if an unobserved hyperedge $e = \{v_1, v_2, \ldots, v_M \}$ exists i.e. $e \not\in E'$ but $e \in E$ where all $v_i \in V$ for $i \in \{1,2,\ldots,M\}$, and $\sigma$ is the logistic sigmoid.

Now, for the given hypergraph $H$, let $\mY^{H} \in \{0,1\}^{|P^\star(V) \setminus \emptyset|}$ be the target {\em random variables} associated with the vertex power set of the graph. 
Let $\cY^{H} \in \{0,1\}^{|P^\star(V) \setminus \emptyset|} $ be the corresponding true values attached to the vertex subsets in the power set, such that $\cY^{H}_e = 1$ iff $e \in E$.
% , i.e. the true hyperedge set.
We then model the joint distribution of the hyperedges in the hypergraphs by making a mean field assumption as:
%  the joint distribution of the hypergraph using a Hypergraph Markov Network 
% is computationally intractable, hence we make a mean field assumption.
% The joint distribution of the vertex and hyperedge set in the graph can then be written as:
\begin{equation}
\small
         P(H) = \prod_{e \in P^\star(V) \setminus \emptyset} \text{Bernoulli}(\mY^{G}_{e} = \cY^{G}_{e}| r(\Gamma(e, H);\Theta))
\end{equation}

Subsequently, to learn the model parameters $\Theta$ - we make a closed world assumption and treat only the observed hyperedges in $E'$ as positive and all other edge as false and seek to maximize the log-likelihood.
\begin{equation}
    \begin{split}
       \Theta = \argmax_{\Theta}\sum\limits_{e \in E'}\log p(\mY^{H}_{e} =1|r(\Gamma(e, H));\Theta) + \\
       \sum_{e \in P^\star(V)\setminus \{E', \emptyset\}}\log p(\mY^{H}_{e}=0|r(\Gamma(e, H));\Theta) 
    \end{split}
\end{equation}

Since the size of vertex power set $(2^{|V|})$, grows  exponentially with the number of vertices, it is computationally intractable to use all negative hyperedges in the training procedure. 
Our training procedure, hence employs a negative sampling procedure (in practice, we use 5 distinct negative samples for every hyperedge in every epoch) combined with a cross entropy loss function, to learn the  model parameters via back-propagation.
This framework can trivially be extended to perform multi class classification on variable sized hyperedges.

\subsection{Hyperedge Completion}
The set expansion task introduced in \cite{zaheer2017deep} makes the infinite de-Finetti assumption i.e. the elements of the set are {\em i.i.d}.
When learning over finite graphs and hypergraphs, this assumption is no longer valid - since the data is relational - i.e. a finite de-Finetti \cite{diaconis1977finite} assumption is required.
Additionally, the partial exchangeability of the structures (adjacency matrix/ incidence matrix) \cite{aldous1981representations} have to be factored in as well.

This raises multiple concerns: (1) computing mutual information of a partial vertex set with all other disjoint vertex subsets in the power set is computationally intractable; (2) to learn a model in the finite de-Finetti setting, we need to consider all possible permutations for a vertex subset.
For example, under the conditions of  finite exchangeability, the point-wise mutual information between two random variables $X, Y$ - where both are disjoint elements of the vertex power set (or hyperedges) i.e. $X, Y \in P^\star(V) \setminus \emptyset, \, X \cap Y = \emptyset$  is given by:
\begin{equation}
    \label{eq:mutual_info}
    s(X|Y) =  \log p(X \cup Y|\, \alpha)  - \log p(X\, |\,\alpha)p(Y|\,\alpha)
\end{equation}
% $$s(X|Y) =  \log p(X \cup Y|\, \alpha)  - \log p(X\, |\,\alpha)p(Y|\,\alpha)$$
where $\alpha$ is a prior and each of $p(X|\alpha), p(Y|\alpha), P(X \cup Y | \alpha)$ cannot be factorized any further i.e.
\begin{equation}
    \label{eq:prob_factorization}
    p(X|\alpha) = \frac{1}{|X|!}{}\sum_{\pi \in \Pi_{X}} \log p(v_{\pi(1)}, v_{\pi(2)}, \ldots , v_{\pi(|X|)}\, | \, \alpha)
\end{equation}
% $$p(X|\alpha) = \frac{1}{|X|!}{}\sum_{\pi \in \Pi_{X}} \log p(v_{\pi(1)}, v_{\pi(2)}, \ldots , v_{\pi(|X|)}\, | \, \alpha)$$
where $v_i \in X, i \in \{1,2,\ldots,|X|\}$ and $\Pi_X$ denotes the set of all possible permutations of the elements of $X$.
The inability to factorize \Cref{eq:prob_factorization} any further, leaves no room for any computational benefits by a strategic addition of vertices - one at a time (i.e. no reuse of computations, whatsoever).

As a solution to this computationally expensive problem, we propose a GAN framework \cite{goodfellow2014generative} to learn a probability distribution over the vertex power set, conditioned on a partially observed hyperedge, without sacrificing on the underlying objective of maximizing  point-wise mutual information between $X, Y$ (\Cref{eq:mutual_info}).
We describe the working of the generator and the discriminator of the GAN, with the help of a single example below.
% as shown in \Cref{fig:gan_model} 

Let $\overline{e}$ denote a partially observed hyperedge and $\Gamma(\overline{e},G)$ denote the representation of the partially observed hyperedge obtained via \Cref{eq:unobserved_struct}.
Let $V_K, \overline{V_{K}}$ denote the true and predicted vertices respectively to complete the partial hyperedge $\overline{e}$, where $V_K, \overline{V_K} \subseteq V \setminus \{\overline{e}\}$.

\subsubsection{Generator($G^\star$)}
The goal of the generator is to accurately predict $\overline{V_K}$ as $V_K$. 
We solve this using a two-fold strategy - first predict the number of elements $K$, missing in the partially observed hyperedge $\overline{e}$ and then jointly select $K$ vertices from $V \setminus \overline{e}$.
Ideally, the selection of the best $K$ vertices should be performed over all vertex subsets of size $K$ (where vertices are sampled from $V \setminus \overline{e}$ without replacement).
However, this is computationally intractable even for small values e.g $K=2,3$ for large graphs with millions of nodes.

We predict the number of elements missing in a hyperedge, $K$, using a function $a_1 : \R^d \to \mathbb{N}$ over the representation of the partial hyperedge, $\Gamma(\overline{e},G)$.
To address the problem of jointly selecting a set of $k$ vertices without sacrificing on computational tractability, we seek to employ a variant of the Top-K problem often used in computing literature.

The standard top-K operation can be adapted to vertices as: given a set of vertices of a graph $\{v_1, v_2, \cdots v_n\} = V \setminus \{\overline{e}\}$, to return a vector $A=\left[A_{1}, \ldots, A_{n}\right]^{\top}$
such that
$$
A_{i}=\left\{\begin{array}{ll}
1, & \text { if } v_{i} \text { is a top- } K \text { element in } V \setminus \overline{e} \\
0, & \text { otherwise. }
\end{array}\right.
$$
However a standard top-K procedure, which operates by sampling vertices (from the vertex set - a categorical distribution) is discrete and hence not differentiable.
% To circumvent this issue, 
% Moreover, theoretically, a universal top-K procedure requires that it captures all interactions between the vertices of the partial hyperedge and the proposed expansion with respect to the entire graph,
% minimally requires a k-WL (where k= $K + |\overline{e}|$) procedure.-
% which introduces unintended computational overheads.
% 
% , so we adapt ancestral gumbel top-k sampling in our procedure.
To alleviate the issue of differentiability, Gumbel softmax \cite{jang2016categorical, maddison2016concrete} could be employed to provide a differentiable approximation to sampling discrete data. 
However, explicit top-K Gumbel sampling (computing likelihood for all possible sets of size $k$ over the complete domain) is computationally prohibitive and hence finds limited applications in hypergraphs with a large number of nodes and hyperedges.

In this work, we sacrifice on differentiability and focus on scalability.
We limit the vertex pool (which can complete the hyperedge) to only vertices in the two hop neighborhood (in the clique expansion $C_H$) of the vertices in the partial hyperedge.
For real world datasets, even the reduced vertex pool consists of a considerable number of vertices - and explicitly computing all sets of size $k$ is still prohibitive.
In such cases, we sample uniformly at random a large number of distinct vertex subsets of size $k$ from the reduced vertex pool, where $k$ is the size predicted by the generator.
In practice, the large number is typically min($\binom{P}{k}$, 100,000), where $P$ is the number of vertices in the reduced vertex pool.
Subsequently, we compute the inner product of the representations of these subsets (computed using \Cref{eq:unobserved_struct}) with the  representation of the partially observed hyperedge.
We then use a simple Top-1 to select the set of size $k$ which maximizes the inner product.

\subsubsection{Discriminator($D^\star$)}
The goal of the discriminator is to distinguish the true, but unobserved hyperedges from the others.
To do this, we obtain representations of $\Gamma(\overline{e},G), \Gamma(V_K,G), \Gamma(\overline{e}\cup V_K,G)$ (and similarly for the predicted $\overline{V_{K}}$ using the generator $G^\star$) and employ the discriminator in the same vein as \Cref{eq:mutual_info}.
As a surrogate for the log-probablities, we learn a function $g: \R^d \to \R^d$ over the representations of $\Gamma(\overline{e},G), \Gamma(V_K,G), \Gamma(\overline{e}\cup V_K,G)$ (log-probabilities in higher dimensional space).
Following this, we apply a function $f:\R^d \to \R^{+} \cup \{0\}$, as a surrogate for the mutual information computation.
The equation of discriminator can then be listed as:

% $$D^\star(V_K | e', G) : f(\Gamma(e'\cup V_K,G) - \Gamma(e',G) - \Gamma(V_K,G)) \to \{0,1\}$$
\begin{equation}
\small
\begin{split}
    \label{eq:discriminator}
    D^\star(V_K | \overline{e}, H) = \sigma(f(g(\Gamma(\overline{e}\cup V_K,H)) \\
    - g(\Gamma(\overline{e},H)) - g(\Gamma(V_K,H))))
\end{split}
\end{equation}

% $$D^\star(V_K | \overline{e}, H) = f(\Gamma(\overline{e}\cup V_K,H) - \Gamma(\overline{e},H) - \Gamma(V_K,H))$$

and correspondingly for $D^\star(G^\star(\overline{V_{K}} | \overline{e}, G))$, where $\sigma$ is the logistic sigmoid.

Our training procedure for the GAN, over the hypergraph $H$, can then be summarized as follows. 
Let $V^\dagger$ denote the value function and let $E'$ denote a set of partial hyperedges and $E$ denote the corresponding set with all hyperedges completed. 
Let ${V_{K_{\overline{e}}}}, \overline{V_{K_{\overline{e}}}}$ denote the corresponding true and predicted vertices to complete the hyperedge.
The value function can then be written as:
\begin{equation}
\small
    \begin{split} \min_{G^\star}\max_{D^\star}V^\dagger(D^\star,G^\star) =
        \sum_{e' \in E'}\log D^\star({V_{K_{e'}}} | e', H) + \\
        \log (1 - D^\star(G^\star(e', H)))
    \end{split}
\end{equation}

In practice, the model parameters of the GAN are learnt using a cross entropy loss and back-propagation.
An MSE loss is employed to train the function $a_1$, the function that predicts the number of missing vertices in a hyperedge, using ground truth information about the number of missing vertices in the partial hyperedge.

%related
% \input{section_files/related.tex}

%results
% \vspace{-0.1in}
\section{Results}
\label{sec:results}
We first briefly describe the datasets and then present our experimental results on the two hypergraph tasks.
\subsection{Datasets}
We use the publicly available hypergraph datasets from \cite{benson2018simplicial} to evaluate the proposed models against multiple baselines (described below).
We ignore the timestamps in the datasets and only use  unique hyperedges which contain greater than 1 vertex.
Moreover, none of the datasets have node or hyperedge features.
We summarize the dataset statistics in the Supplementary material.
We briefly describe the hypergraphs and the hyperedges in the different datasets below.

\begin{itemize}[noitemsep]
    \item Online tagging data (tags-math-sx; tagsask-ubuntu). In this dataset, nodes are tags (annotations) and a hyperedge is a set of tags for a question on online Stack Exchange forums.
    \item Online thread participation data (threads-math-sx; threads-ask-ubuntu): Nodes are users and a hyperedge is a set of users answering a question on a forum.
    \item Two drug networks from the National Drug Code Directory, namely (1) NDC-classes: Nodes are class labels and a hyperedge is the set of class labels applied to a drug (all applied at one time) and (2) NDC-substances: Nodes are substances and a hyperedge is the set of substances in a drug.
    \item US. Congress data (congress-bills): Nodes are members of Congress and a hyperedge is the set of members in a committee or cosponsoring a bill.
    \item Email networks (email-Enron; email-Eu): Nodes are email addresses and a hyperedge is a set consisting of all recipient addresses on an email along with the sender’s address.
    \item Contact networks (contact-high-school; contact-primary-school): Nodes are people and a hyperedge is a set of people in close proximity to each other.
    \item Drug use in the Drug Abuse Warning Network (DAWN): Nodes are drugs and a hyperedge is the set of drugs reportedly used by a patient before an emergency department visit.
\end{itemize}

\begin{table*}[ht!!!]
\vspace{-0.65in}
\begin{threeparttable}[b]
\caption{F1 scores for the hyperedge classification task (Higher is better).}
\label{tab:classification}
% \resizebox{1\textwidth}{!}{
\begin{tabular}{l lllll}\toprule
                       & \multicolumn{1}{c}{Trivial} & \multicolumn{1}{c}{\begin{tabular}[c]{@{}c@{}}Clique Expansion-\\ GCN \\ (HGNN)\end{tabular}} & \multicolumn{1}{c}{\begin{tabular}[c]{@{}c@{}}Clique Expansion-\\ SAGE\end{tabular}} & \multicolumn{1}{c}{\begin{tabular}[c]{@{}c@{}}Star Expansion - \\ Heterogenous -\\ RGCN\end{tabular}} & \multicolumn{1}{c}{Ours} \\ \midrule
NDC-classes            & 0.286                       & 0.614(0.005)                                                                                      & 0.657(0.020)                                                                         & 0.676(0.049)                                                                                          & \textBF{0.768(0.004)}            \\
NDC-substances         & 0.286                       & 0.421(0.014)                                                                                      & 0.479(0.007)                                                                         & \textBF{0.525(0.006)}                                                                                        &     \textBF{0.512(0.032)}                     \\
DAWN                   & 0.286                       &      0.624(0.010)                                                                                             & 0.664(0.006)                                                                         &         0.634(0.003)                                                                                              &              \textBF{0.677(0.004)}            \\
contact-primary-school & 0.286                       & 0.645(0.031)                                                                                      & 0.681(0.014)                                                                         & 0.669(0.012)                                                                                          & \textBF{0.716(0.034)}             \\
contact-high-school    & 0.286                       & \textBF{0.759(0.030)}                                                                                      & 0.724(0.009)                                                                         & 0.739(0.012)                                                                                          & \textBF{0.786(0.033)}            \\
tags-math-sx           & 0.286                       &   0.599(0.009)    &    \textBF{0.635(0.003)}      &    0.572(0.003)                  &    \textBF{0.642(0.006)}                      \\
tags-ask-ubuntu        & 0.286                       &     0.545(0.005)                                                                                              & 0.597(0.007)                                                                         &      0.545(0.006)                                                                                                 &   \textBF{0.605(0.002)}                       \\
threads-math-sx        & 0.286                       &     0.453(0.017)                     &   0.553(0.012)               &       0.487(0.006)                 &           \textBF{0.586(0.002)}               \\
threads-ask-ubuntu     & 0.286                       & 0.425(0.007)                                                                                      & \textBF{0.512(0.007)}                                                                         & 0.464(0.010)                                                                                          & 0.488(0.012)             \\
email-Enron            & 0.286                       & 0.618(0.032)                                                                                      & 0.594(0.046)                                                                         & 0.599(0.040)                                                                                          & \textBF{0.685(0.016)}             \\
email-EU               & 0.286                       & 0.664(0.003)                                                                                      & 0.651(0.019)                                                                         & 0.661(0.006)                                                                                          &      \textBF{0.687(0.002)}                    \\
congress-bills         & 0.286                       & 0.412(0.003)                                                                                      & 0.530(0.055)                                                                         &           0.544(0.004)                                                                                            &       \textBF{0.566(0.011)}                    \\ 
% Average Rank & 5 & & & &\\
\bottomrule
\end{tabular}
% }
\begin{tablenotes}
\small
\item [1] A 5-fold cross validation procedure is used - numbers outside the parenthesis are the mean values and the standard deviation is specified within the parenthesis
\item [2] Bold values show maximum empirical average, and multiple bolds happen when its standard deviation overlaps with another average.
% \item [3] Inductive Results are provided in the Appendix
\end{tablenotes}
\end{threeparttable}
\vspace{-0.2in}
\end{table*}

\begin{table}[t!!!]
\resizebox{0.5\textwidth}{!}{
\begin{threeparttable}[b]
\caption{Normalized Set Difference scores for the hyperedge expansion task (lower is better)}
\label{tab:expansion}

\begin{tabular}{l c@{\hspace{4pt}}c@{\hspace{4pt}}c}\toprule
                       & \multicolumn{1}{c}{Simple} & \multicolumn{1}{c}{Recursive} & \multicolumn{1}{c}{Ours} \\ \midrule
NDC-classes            & 1.207(0.073)                      & 1.163(0.015)                         &      \textBF{1.107(0.007)}                    \\
NDC-substances         &   1.167(0.000)                         &       \textBF{1.161(0.009)}                       &       \textBF{1.153(0.004)}                   \\
DAWN                   &    1.213(0.006)                        &   1.197(0.022)                            &   \textBF{1.088(0.018)}                       \\
contact-primary-school & 0.983(0.006)                      & 0.986(0.001) &            \textBF{0.970(0.005)}              \\
contact-high-school    & \textBF{0.990(0.014)}                      & 1.000(0.000)                       &   \textBF{0.989(0.001)}                        \\
tags-math-sx           &    1.012(0.025)                        &     1.003(0.014)                          &      \textBF{0.982(0.011)}                    \\
tags-ask-ubuntu        &      1.008(0.003)                      &     1.005(0.003)                          &      \textBF{0.972(0.001)}                    \\
% threads-math-sx        &     1.002(0.006)                       &                               &                          \\
threads-ask-ubuntu     & 0.999(0.000)                      &     0.999(0.000)                          & \textBF{0.981(0.003)}                     \\
email-Enron            & 1.152(0.045)                      & 1.182(0.015)                         &    \textBF{1.117(0.049)}                      \\
email-EU               &      1.199(0.002)                      &   1.224(0.010)                            &         \textBF{1.116(0.013)}                 \\
congress-bills         &      1.186(0.004)                      &       1.189(0.001)                        &    \textBF{1.107(0.004)}                    \\ 
\bottomrule
\end{tabular}
\begin{tablenotes}
\item [1] A 5-fold cross validation procedure is used - numbers outside the parenthesis are the mean values and the standard deviation is specified within the parenthesis
\item [2] Bold values show minimum empirical average, and multiple bolds happen when its standard deviation overlaps with another average.
% \item [3] Inductive Results are provided in the Appendix
\end{tablenotes}
\end{threeparttable}
}
\vspace{-0.5in}
\end{table}
% 
% \subsubsection{Hyperedge Completion}
% \begin{table}[]
% \resizebox{1.25\textwidth}{!}{
% }
% \end{table}
\subsection{Experimental Results}
\paragraph{Hyperedge Classification}
In this task, we compare our model against five baselines.
The first is a trivial predictor, which always predicts 1 for any hyperedge (in practice, we use 5 negative samples for every real hyperedge).
The second two baselines utilize a GCN \cite{kipf2016semi} or GraphSAGE \cite{hamilton2017inductive} on the clique expansion of the hypergraph.
GCN on the clique expansion on the hypergraph is the model proposed by \cite{feng2019hypergraph} as HGNN.
For the fourth baseline, we utilize the star expansion of the hypergraph - and employ a heterogeneous RGCN to learn the vertex, hyperedge embeddings.
In each of the baselines, unobserved hyperedge embeddings are  obtained by aggregating the representations of the vertices it contains, using a learnable set function \cite{zaheer2017deep, murphy2018janossy}.
We report F1 scores on the eight datasets in \Cref{tab:classification}.
More details about the experimental setup is presented in the Supplementary material.

\paragraph{Hyperedge Expansion}
Due to lack of prior work in hyperedge expansion, here we compare our strategy against two other baselines for hyperedge expansion (with the an identical GAN framework and setup to predict the number of missing vertices, albeit without computing joint representations of predicted vertices) : (1) Addition of Top-K vertices, considered independently of each other (2) Recursive addition of Top-1 vertex.
% Again, we employ \cref{eq:unobserved_struct} to learn structural representations of partially observed hypereges.
% Following this we use our GAN model to predict both the size and the missing elements in the hyperedge.
Since all the three models are able to accurately (close to 100\% accuracy) predict the number of missing elements, we introduce {\em normalized set difference}, as a statistic to compare the models.
Normalized Set difference (permutation invariant) is given by the number of insertion/ deletions/ modifications required to go from the predicted completion to the target completion divided by the number of missing elements in the target completion.
For example, let \{7,8,9\} be a set which we wish to expand. 
Then the normalized set difference between a predicted completion \{3,5,1,4\} and target completion \{1,2\} is computed as by (1+2)/2 = 1.5 (where there is 1 modification and 2 deletions).
It is clear to see that, a lower normalized set difference score is better and a score of 0 indicates a perfect set prediction.
Results are presented in \Cref{tab:expansion}.
% \vspace{-0.1in}
\section{Discussion}
In the hyperedge classification task, from \Cref{tab:classification} it is clear to see that our model which with provable expressive properties performs better than the baselines, on most datasets.
All three non-trivial baselines appear to suffer from their inability to  capture higher order interactions between the vertices in a hyperedge.
Moreover, the loss in information by using a proxy graph - in the form of the clique expansion - also affects the performance of the SAGE and GCN baselines.
The SAGE baseline obtaining better F1 scores over GCN suggests that the self loop introduced between vertices in the clique expansion appears to hurt performance.
The lower scores of the star expansion models can be attributed to its inability in capturing vertex-vertex and hyperedge-hyperedge interactions.

For the hyperedge expansion task, from \Cref{tab:expansion} it is clear to see that adding vertices in a way which captures interactions amongst them performs better than adding vertices independently of each other or in a recursive manner.
The relatively weaker performance of adding vertices recursively, one at a time can be attributed to a poor choice of selection of the first vertex to be added (once an unlikely vertex is added, the sequence cannot be corrected).

%conclusions
% \vspace{-0.1in}
\section{Conclusions}
\label{sec:conclusions}
In this work, we developed a hypergraph neural network to learn provably expressive representations of vertices, hyperedges and the complete hypergraph.
We  proposed frameworks for hyperedge classification and a novel hyperedge expansion task, evaluated performance on multiple real-world hypergraph datasets, and demonstrated consistent, significant improvement in accuracy, over state-of-the-art models.

\bibliography{main}
\bibliographystyle{plain}

\clearpage

\section{APPENDIX}
\subsection{Examples:}
\label{subsection:appendix}
Let $\cA$ denote the complete set of substances which are possible components in a prescription drug. Now, given a partial set of substances $X' \subset \cA$ part of a single drug, the hyperedge expansion entails  completing the set $X'$ as $X$ with a set of substances from $\cA$, (which were unobserved due to the data collection procedure for instance), with the set $X$ chosen {\em s.t.} 
$X = \argmax_{X' \subseteq X \subseteq \cA} p_{data}(X) - p_{data}(X')$.
On the other hand, an example of a hyperedge classification tasks involves determining whether a certain set of substances can form a valid drug or alternatively classifying the nature of a prescription drug.
From the above examples, it is clear to see that the hyperedge expansion and hyperedge classification necessitate the framework to jointly capture dependencies between all the elements of an input set (for instance, the associated outputs in these two tasks, requires us to capture all interactions between a set of substances, rather than just the pairwise interactions between a single substances and its neighbors computed independently - as in node classification) and hence are classed as higher order tasks.
Additionally, for the hyperedge expansion task, the associated output is a finite set and hence in addition to maximizing the interactions between the constituent elements it is also required to be permutation invariant.
For instance, in the expansion task, the training data $X \in \cX$ as well as the associated target variable $Y \in \cY$ to be predicted are both sets. 
The tasks are further compounded by the fact that the training data and the outputs are both relational i.e. the representation of a vertex/ hyperedge also depends on other sets (composition of other observed drugs) in the family of sets i.e. the data is {\em non i.i.d}.

% \subsection{Group Theory Preliminaries}
% We present some preliminaries from group theory, which supplement the ones presented in the main paper.

\subsection{Proofs of Properties, Remarks and Theorems}
We restate the properties, remark and theorems for convenience and prove them.

\begin{property}[Vertex Representations]
The representation of a vertex $v \in V$ in a hypergraph $H$ learnt using $\Cref{eq:vertex_update}$ is a $\cG$-invariant representation $\Phi(v, V, E, \mX, \mE)$ where $\Phi:V \times \Sigma_{n,m} \to \R^d, d\geq 1$ such that  $\Phi(v, V, E, \mX, \mE)$ = $\Phi((\pi_1(v), \pi_1(V), \pi_2(\pi_1(E)), \pi_1(\mX), \pi_2(\pi_1(\mE)))$ $\forall \pi_1 \forall \pi_2$ where $ \pi_1 \in \mathbb{S}_n \, \text{and} \, \pi_2 \in \mathbb{S}_m$.  
Moreover, two vertices $v_1, v_2$ which belong to the same vertex equivalence class i.e. $v_1 \cong v_2$ obtain the same representation.
\end{property}
\begin{proof}
Part 1: Proof by contradiction.
Let $\pi, \pi' \in \mathbb{S}_n$ be two different vertex permutation actions and let  $\Phi(\pi(v), \pi(V), E, \pi(\mX), \mE)$ $\not=$ $\Phi(\pi'(v), \pi'(V), E, \pi'(\mX), \mE)$.
This implies that the same node gets different representations based on an ordering of the vertex set.
From $\cref{eq:vertex_update}$ it is clear to see that the set function $g^k$ ensures that the vertex representation is not impacted by the edge permutation action.
Now let $k=1$
Expanding $\cref{eq:vertex_update}$ for both vertex permutation actions and applying the cancellation law of groups, $h_v^1$ is independent of the permutation action.
Since $h_v^{0}$ is identical for both, it means the difference arises from the edge permutation action, which is not possible.
Now, we can show using induction, the contradiction holds for a certain $k, k \geq 2$, then it holds for $k+1$ as well.
Hence, $\Phi(\pi(v), \pi(V), E, \pi(\mX), \mE)$ $=$ $\Phi(\pi'(v), \pi'(V), E, \pi'(\mX), \mE)$

Part2: Proof by contradiction
Let $v_1, v_2' \in V$ be two isomorphic vertices and let  $\Phi(v_1, V, E, \mX, \mE)$ $\not=$ $\Phi(v_2, V, E, \mX, \mE)$
This implies $h_{v_1}^k \not= h_{v_2}^k \, \forall k \geq 0$
However, by the definition, the two vertices are isomorphic, i.e. they have the same initial node features (if available) i.e. $h_{v_1}^0 = h_{v_2}^0$ and they also posses an isomorphic neighborhood.
\Cref{eq:vertex_update} is  deterministic, hence the representations obtained by the vertices $v_1, v_2$ are also identical after 1 iteration i.e. $h_{v_1}^1 = h_{v_2}^1$ .
Now using induction we can show that, the representations for $h_{v_1}^{k}$ is the same as $h_{v_2}^{k}$ for any $k \geq 2$
Hence $\Phi(v_1, V, E, \mX, \mE)$ $=$ $\Phi(v_2, V, E, \mX, \mE)$ when $v_1 \cong v_2$
\end{proof}
\begin{property}[Hyperedge Representations]
The representation of a hyperedge $e \in E$ in a hypergraph $H$ learnt using $\Cref{eq:edge_update}$ is a $\cG$-invariant representation $\Phi(e, V, E, \mX, \mE)$ where $\Phi:P^\star(V) \times \Sigma_{n,m} \to \R^d, d\geq 1$ such that  $\Phi(e, V, E, \mX, \mE)$ = $\Phi((\pi_2(\pi_1(e)), \pi_1(V), \pi_2(\pi_1(E)), \pi_1(\mX), \pi_2(\pi_1(\mE)))$ $\forall \pi_1 \forall \pi_2$ where $ \pi_1 \in \mathbb{S}_n \, \text{and} \, \pi_2 \in \mathbb{S}_m$
Moreover, two hyperedges $e_1, e_2$ which belong to the same hyperedge equivalence class i.e. $e_1 \cong e_2$ obtain the same representation.
\end{property}
\begin{proof}
Proof is similar to the two part $\cG$-invariant vertex representation proof given above.
Replace the vertex permutation action with a joint vertex edge permutation action and similarly use the cancellation law of groups twice.
\end{proof}

\begin{theorem}[\cite{tyshkevich1996line}]
Let $H_1, H_2$ be hypergraphs without isolated vertices whose line hypergraphs ${LG}_{H_1}, {LG}_{H_2}$ are isomorphic. 
Then $H_{1} \cong H_{2}$ if and only if there exists a bijection $\beta: V {LG}_{H_{1}} \rightarrow V {LG}_{H_{2}}$ such that $\beta\left(S_{H_{1}}\right)=S_{H_{2}},$ where $S_{ H_{i}}$ is the family of stars of the hypergraph $H_{i}$
\end{theorem}

\begin{proof}
Theorem is a direct restatement of the theorem in the original work.
Please refer to \cite{tyshkevich1996line} for the proof.
\end{proof}
\begin{theorem}
Let $H_1$, $H_2$ be two non isomorphic hypergraphs with finite vertex and hyperedge sets and no isolated vertices.
If the Weisfeiler-Lehman test of isomorphism decides their line graphs $L_{H_1}, L_{H_2}$ or the star expansions of their duals  $H_1^\star, H_2^\star$ to be not isomorphic 
% or the clique expansions $C_{H_1}, C_{H_2}$ to be not isomorphic, 
then there exists a function $\Gamma: \Sigma_{n,m} \to \R^d$ (via \Cref{eq:separate}) and parameters $\Theta$ that maps the hypergraphs $H_1, H_2$ to different representations.
% Two non-isomorphic hypergraphs $H_1$, $H_2$, without isolated vertices are distinguishable  if either of their clique expansions or line graphs are 1-WL distinguishable.
% Isomorphic hypergraphs also obtain the same representation.
\end{theorem}
\begin{proof}
Part 1: Proof by construction, for the line graph $L_H$.
Consider \Cref{eq:edge_update}.
By construction, make the set function $p$ as an injective function with a multiplier of a negligible value i.e. $\to 0$.
This implies, a hyperedge only receives information from its adjacent hyperedges.
Since we use injective set functions, following the proof of \cite{xu2018powerful} Lemma 2 and Theorem 3, by induction it is easy to see that if the 1-WL isomorphism test decides that the line graphs are non-isomorphic, the representations obtained by the hyperedges through the iterative message passing procedure are also different.

Part 2: Proof by construction, for the dual graph $H^\star$
Again, consider \Cref{eq:edge_update}.
By construction, associate a unique identifier with every node and hyperedge in the hypergraph.
Construct $p$ as an identity map, this implies, a hyperedge preserves information from which vertices it receives information as well.
Since the above $p$ is injective, again following the proof of \cite{xu2018powerful}  Lemma 2 and Theorem 3, by induction it is easy to see that if the 1-WL isomorphism test decides that the dual of a hypergraph are non-isomorphic, the representations obtained by the hyperedges through the iterative message passing procedure are also different.

Part 3: From part 1 and part 2 of the proof above, we see that if either the line graphs or the dual of the hypergraphs are distinguishable by the 1-WL isomorphism test as non-isomorphic then our proposed model is able to detect it as well.
From the property of vertex representations it also seen that isomorphic vertices obtain the same representation - hence preserving the family of stars representation as well.
Now consider \Cref{eq:separate}
Now, if the line hypergraphs $LG_{H_1}$ and $LG_{2}$ are distinguishable via the line graphs or the dual graphs then the representation obtained by hyperedge aggregations are different.
Correspondingly if the family of stars - does not preserve a bijection across the two hypergraphs, then the representation of the graphs are distinguishable using the vertex aggregation.
\end{proof}

\begin{corollary}
There exists a function  $\Gamma : P^\star (V) \times \Sigma_{n,m} \to \R^d $ (via \Cref{eq:unobserved_struct}) and parameters $\Theta$ that map two non-isomorphic hyperedges $e_1, e_2$ to different representations.
\end{corollary}
\begin{proof}
Proof is a direct consequence of the above theorem, \cref{eq:unobserved_struct} and above property of hyperedges.
\end{proof}
\begin{remark}[Separate Exchangeability]
The representation of a hypergraph $H$ learnt using the function $\Gamma: \Sigma_{n,m} \to \R^d$ (via \Cref{eq:separate}) preserves the separate exchangeability of the incidence structure $\mI$ of the hypergraph.
\end{remark}
\begin{proof}
From \Cref{eq:separate}, it is clear that once the representations of the observed vertices and hyperedges are obtained, the vertex permutation actions don't affect the edge permutation and vice versa - i.e. the set functions $\phi, \rho$ act independently of each other.
From \Cref{eq:separate} and through the use of set functions, it is also clear that the representation of the hypergraph is invariant to permutations of both vertex and edge.
\end{proof}

\subsection{Dataset Statistics}
In \Cref{tab:dataset} we list the number of vertices and hyperedges for each of the datasets we have considered. 

\begin{table}[h!!!!]
\caption{Dataset Statistics}
\label{tab:dataset}
\begin{tabular}{l ll}\toprule
                       & \multicolumn{1}{c}{\# Vertices} & \multicolumn{1}{c}{\# Hyperedges} \\ \midrule
NDC-classes            & 1161                            & 679                               \\
NDC-substances         & 5556                            & 4916                              \\
contact-primary-school & 242                             & 4036                              \\
contact-high-school    & 327                             & 1870                              \\
threads-math-sx        & 201863                          & 177398                            \\
threads-ask-ubuntu     & 200974                          & 18785                             \\
email-Enron            & 148                             & 577                               \\
email-EU               & 1005                            & 10631         \\
\bottomrule
\end{tabular}
\end{table}

\subsection{Experimental Setup}
Our implementation is in PyTorch using Python 3.6.
For the hyperedge classification task, we used 5 negative samples for each positive sample.
For the hyperedge expansion task, the number of vertices to be added varied from 2 to 7.
The implementations for graph neural networks are done using the Deep Graph Library \cite{wang2019dgl}.
We used two convolutional layers for all the baselines as well as our model since it had the best performance in our tasks (we had tested with 2/3/4/5 convolutional layers).
For all the models, the hidden dimension for the convolutional layers, set functions was tuned from \{8,16,32,64\}.
Optimization is performed with the Adam Optimizer and the learning rate was tuned in \{0.1, 0.01, 0.001, 0.0001, 0.00001\}.
For the set functions we chose from \cite{zaheer2017deep} and \cite{murphy2018janossy}.
For more details refer to the code provided.

\end{document}